\documentclass{article}

\usepackage[numbers,sort]{natbib}
\usepackage[T1]{fontenc}
\usepackage{lmodern}

\usepackage{enumitem}

\usepackage{amsmath}
\usepackage{amssymb}
\usepackage{mathtools}
\usepackage{amsthm}

\usepackage{caption}
\usepackage{subcaption}

\usepackage{algorithm}
\usepackage{algpseudocode}

\DeclareMathOperator*{\dis}{dis} 
\newcommand{\ddist}[1]{{L_{#1}}}
\newcommand{\tildeddist}[1]{{\tilde L_{#1}}}

\usepackage[utf8]{inputenc} 
\usepackage[T1]{fontenc}    
\usepackage{hyperref}       
\usepackage{url}            
\usepackage{booktabs}       
\usepackage{amsfonts}       
\usepackage{nicefrac}       
\usepackage{microtype}      
\usepackage{xcolor}         

\theoremstyle{plain}
\newtheorem{theorem}{Theorem}[section]

\theoremstyle{definition}

\theoremstyle{remark}

\begin{document}

\title{Orthogonalization of data via Gromov-Wasserstein type feedback for clustering and visualization}

\author{
  Martin Ryner \\
  Vironova AB, Stockholm, Sweden\\
  Division of Optimization and Systems Theory, \\
  Department of Mathematics,\\ 
  KTH Royal Institute of Technology, Stockholm, Sweden\\
  \texttt{martin.ryner@vironova.com, martinrr@kth.se} \\
  \\
  Johan Karlsson\\
  Division of Optimization and Systems Theory, \\
  Department of Mathematics, 
  \\KTH Royal Institute of Technology, Stockholm, Sweden\\
   \texttt{johan.karlsson@math.kth.se}  
}

\maketitle

\begin{abstract}
In this paper we propose an adaptive approach for clustering and visualization of data by an orthogonalization process.
Starting with the data points being represented by a Markov process using the diffusion map framework, the method adaptively increase the orthogonality of the clusters by applying a feedback mechanism inspired by the Gromov-Wasserstein distance.
This mechanism iteratively increases the spectral gap and refines the orthogonality of the data to achieve a clustering with high specificity. 
By using the diffusion map framework and representing the relation between data points using transition probabilities, the method is robust with respect to both the underlying distance, noise in the data and random initialization. We prove that the method converges globally to a unique fixpoint for certain parameter values.
We also propose a related approach where the transition probabilities in the Markov process are required to be doubly stochastic, in which case the method generates a minimizer to a nonconvex optimization problem.
We apply the method on cryo-electron microscopy image data from biopharmaceutical manufacturing where we can confirm biologically relevant insights related to therapeutic efficacy. We consider an example with morphological variations of gene packaging and confirm that the method produces biologically meaningful clustering results consistent with human expert classification.  

\end{abstract}

\section{Introduction}

Many problems in data science amounts to simplify and categorize a large set of unlabeled data points, and there is vast literature on this topic ranging from classical approaches such as K-means clustering \cite{macqueen1967some} and hierarchical clustering \cite{johnson1967hierarchical}, to more recent approaches such as spectral clustering \cite{ng2001spectral}, clustering using diffusion maps \cite{coifman2006diffusion}, topological data analysis \cite{carlsson2009topology}, and dimensionality reduction by minimizing the relative entropy \cite{hinton2002stochastic,van2008visualizing,mcinnes2018umap}.
Clustering or segmentation of a set of data points depends on the relationship between the points, which is noisy in many practical applications. 
Classical methods such as single linkage has high resolution and provides sharp decision support for ideal data \cite{johnson1967hierarchical}
but is highly sensitive to noise as it does not involve dynamics for noise cancellation, however there are attempts to remedy this \cite{ros2019hierarchical}. 
Averaging methods, on the other hand, such as clustering based on Diffusion maps and spectral clustering are more robust, but as the diffusion is a stochastic process it may result in lower resolution because of the dimensional reduction of the transition map.

In order to improve the resolution and the specificity of the diffusion methods, several approaches have been proposed to increase the spectral gap and to remove noise in the transition matrix.
In particular, optimization to create low rank approximations of the transition matrix has been used utilizing graph cut functionals, entropy functionals, and norms  \cite{zass2006doubly,wang2010learning,beauchemin2015affinity}. Another such concept that has been successfully used is doubly stochastic normalization, 
\cite{zass2006doubly, landa2021doubly}.

In this paper we propose a new approach that takes the diffusion map framework as a starting point
 and iteratively increases the spectral gap in the transition matrix by using feedback inspired by the cost in the Gromov-Wasserstein distance.
This corresponds to an orthogonalization of the transition matrix and its corresponding diffusion distance matrix.
It therefore gradually increases the spectral gap between the dominant eigenvalues representing each cluster and the small eigenvalues representing internal variations, as the Markov kernel gets increasingly non-ergodic. We show certain convergence results of the method and compare it with state-of-the-art on two examples.

\subsection{Our contribution}
In this paper we provide: 
\begin{itemize}[noitemsep, topsep = 0pt]
\item A framework for adaptive learning consisting of adaptive classification methods that is both robust to noise and provides clear separation and decision support.
\item We show the method has certain desirable properties, such as stability and unique solution under certain conditions. 
\item We show that one of the iterative methods can be seen as a nonlinear optimization problem.
\item 
Using our method on biomedical data, subgroups in the data become accumulated in a way that captures essential features consistent with expert classification. 
\end{itemize}

\section{Background}

In this section we will highlight relevant previous work and set up notation. 

\subsection{Diffusion maps and spectral clustering}\label{ssec:diffusion_maps}

Coifman and Lafon introduced a framework for diffusion maps \cite{coifman2006diffusion}  where diffusion processes are used to find geometric descriptions of data sets. In particular, they study the corresponding Markov matrices 
and their spectral properties to construct coordinates, or so-called diffusion maps, for  representing of geometric structures.  Specifically, a transition kernel $k_{\epsilon,\alpha}$ is generated by a positive but monotonically decreasing function on the distance $d$ between two points $x$ and $y$ in a compact set $X$ (typically a discrete set or a box in $\mathbb{R}^n$), e.g., $
k_{\epsilon,\alpha}: X^2 \rightarrow \mathbb{R}^+$,
\begin{equation}
k_{\epsilon,\alpha}(x,y) = \frac{ e^{-d(x,y)^2/\epsilon}}{q_\epsilon^\alpha(x)q_\epsilon^\alpha(y)}
\label{eq:normalization}
\end{equation}
where
$q_{\epsilon}(x) = \int_y e^{-d(x,y)^2/\epsilon} q(y) dy$ and $\alpha\ge 0$ is given. 
To maintain mass, the kernel is normalized: 
\begin{equation}
p_{\epsilon,\alpha}(x,y) = k_{\epsilon,\alpha}(x,y)/d_{\epsilon,\alpha}(x)
\label{eq:markovprocess}
\end{equation}
where $d_{\epsilon,\alpha}(x) = \int_{y\in X} k_{\epsilon,\alpha}(x,y)dy$.

Let $P$ be an averaging or diffusion operator of a function $f$, defined by
\[Pf(x) = \int_{y\in X}p_{\epsilon,\alpha}(x,y)f(y)dy,\] 
and the corresponding dual operator $M$ is a Markov operator, given by  $Mv(y) = \int_{x \in X}p_{\epsilon,\alpha}(x,y)v(x)dx$.
Since the kernel $k_{\epsilon,\alpha}$ is symmetric, the stationary distribution of $M$ is given by $\pi(x)= d_{\epsilon,\alpha}(x)/\int_{x} d_{\epsilon,\alpha}(x) dx$,
and thus  
\begin{equation*}
a(x,y) = \frac{\sqrt{\pi(x)}}{\sqrt{\pi(y)}}
p_{\epsilon,\alpha}(x,y) 
\end{equation*}
is also a symmetric kernel.
If $\pi(x) > 0$ for all $x$ then $a$ is a self adjoint bounded operator  and it can be written $a(x,y) = \sum_l \lambda_l\phi_l(x)\phi_l(y)$, where $\{\phi_i\}_i$ is an ON-basis. With this relationship we can write 
$p_{\epsilon,\alpha}(x,y) =\sum_l \varphi_l(y) \lambda_l\psi_l(x)$ where $\psi_i(x) = \phi_i(x)/\sqrt{\pi(x)}$ and $\varphi_i(y) = \phi_i(y)\sqrt{\pi(y)}$. 
Note here that for the maximum modulus eigenvalue $\lambda_1 = 1$, we have $\phi_1 = \sqrt{\pi}$ which gives $\psi_1 \equiv 1$ and $\varphi_1 = \pi$.

Furthermore Coifman and Lafon introduced a diffusion time $t$ on $p_1:=p_{\epsilon,\alpha}$ so that $p_t(x,y) := \sum_l \varphi_l(y) \lambda_l^t\psi_l(x)$ which gives rise to an idea of a diffusion coordinate $X_t(x) = \left\lbrace\lambda_l^t\psi_l(x)\right\rbrace_l$ with a corresponding diffusion distance
\begin{equation}
\label{eq:diffusionDistance}
D_p^t(x,y)  = \|p_t(x,\cdot)-p_t(y,\cdot)\|_{L^2(X,d\mu/ \pi)}=\left(\sum_l\lambda_l^{2t}\left(\psi_l(x)-\psi_l(y)\right)^2\right)^{\frac{1}{2}}.
\end{equation}

Landa, Coifman and Kluger adapted this framework \cite{landa2021doubly}  requiring that the kernels are doubly stochastic and was shown to be more robust to certain types of noise. Moreover, constructing doubly stochastic kernels can be done using entropy regularized optimal mass transport and Sinkhorn iterations \cite{cuturi2013sinkhorn}.

\subsection{The Gromov-Hausdorff distance and the Gromov-Wasserstein distance}\label{ssec:gromov_hausdorff}

The Gromov-Hausdorff distance is a notion of distance that characterizes how far two metric spaces are from being isometric \cite{edwards1975structure,gromov1981groups}.
To define this, consider the concept of correspondence $R$ between two metric spaces $(X,d_X)$ and $(Y,d_Y)$, which is a collection of pairs $(x,y)$ so that every $x\in X$ and $y \in Y $ exist in at least one pair.
More precisely a correspondence is defined by $R = \{(f(z),g(z)): z \in Z\}$, where 
$(Z,d_Z)$ be metric spaces and where $f: Z\rightarrow X$ and $g: Z \rightarrow Y$ are surjective maps. We let $\mathcal{R}$ denote the set of all such correspondences. A distortion $\dis$ of a correspondence is   
\begin{equation}\label{eq:GH}
\dis(R) := \sup\{c(x,y,x',y'); (x,y), (x',y') \in R\}
\end{equation}
where the cost $c(x,y,x',y') = |d_X(x,x')-d_Y(y,y')|$. Then the Gromov-Hausdorff distance can then be defined as  
\[d_{GH} = \frac{1}{2} \inf_{R\in \mathcal{R}} \dis (R).\]
Originally the Gromov-Hausdorff distance was defined using topological spaces, however \eqref{eq:GH} is an equivalent definition when considering metric spaces \cite{burago2001course}. This problem is in general computationally expensive, and  even for the discrete version, this can in many cases be formulated as a quadratic assignment problem which is NP-Hard  \cite{pardalos1994quadratic,memoli2007use}.

The Gromov-Wasserstein (GW) distance is a type of relaxation of \eqref{eq:GH}, and defines a distance $d_{GW,2}(\mu,\nu)$ between the two measures $\mu$ and $\nu$ given by
\begin{align*}
&\min_{\Gamma\ge 0} \left(\int_{x,y,x',y'}  \!\!\!\!\!c(x,y,x',y')^2 \Gamma(x,y)\Gamma(x',y') dx dydx'dy'\!\right)^{\frac{1}{2}}\\
& \mbox{ s.t. } \; \int \Gamma(x,y) dx=  \mu(y) \;\forall y, \int \Gamma(x,y) dy = \nu(x) \;\forall x. 
\end{align*}

Differences, similarities and continuity has been them has been thoroughly evaluated by Mémoli \citep{memoli2007use,memoli2008gromov}. This distance will be used to motivate the orthogonalization functional in Section~\ref{sec:Orthogonalization}.

\subsection{Visualization techniques}

Several visualization techniques, such as SNE \cite{hinton2002stochastic}, t-SNE \cite{van2008visualizing} and U-Map \cite{mcinnes2018umap}, utilize low dimensional function representations for reducing the dimensionality of the data whilst preserving the relations of the points.
Next, we will briefly summarize some main points from these methods, which are considered state-of-the-art for clustering and visualization.

Stochastic neighbor embedding (SNE) is a way to reduce the dimensionality of data by finding the best approximation in the relative entropy sense \cite{hinton2002stochastic}. Here, let $X$ be a finite set and let $P$ be the set of stochastic matrices with dimension $|X|\times |X|$. Given such a stochastic matrix $p\in P$, then we seek 
an approximation  $q$ with the exponential structure (with dimension N)
\[
Q=\left\{q_\Psi(x,y) := \frac{\exp(-\|\Psi(x)-\Psi(y)\|^2/\epsilon)}{\|\exp(-\|\Psi(x)-\Psi(\cdot)\|^2/\epsilon)\|_1}\mid \Psi: X \rightarrow \mathbb{R}^N\right\}
\]
which is as close as possible to $p$ in the relative entropy, i.e., which minimizes 
\begin{equation}\label{eq:SNE}
\min_{q_\Psi\in Q}\sum_{x,y \in X} p(x,y)\log \frac{p(x,y)}{ q_\Psi(x,y)}  
\end{equation}
Note that $q^* = p$ is optimal when $p$ is on the same form and dimensionality as $q$, i.e., when $p\in Q$.
A variation of SNE is t-SNE \cite{van2008visualizing} where one seeks a t-distributed approximation on the rational form
\[
q_\Psi(x,y) := \frac{(1+ \|\Psi(x)-\Psi(y)\|^2)^{-1}}{\|(1+ \|\Psi(x)-\Psi(\cdot)\|^2)^{-1}\|_1}
\]
that minimizes the objective in \eqref{eq:SNE}.

More recently, McInnes et al. introduced the U-Map method \cite{mcinnes2018umap} which also aims to visualize data using locally connected low dimensional manifolds, while minimizing
\[
\min_{q} \sum_{x,y \in X} \left(p(x,y)\log \frac{p(x,y)}{q(x,y)} + (1-p(x,y))\log \frac{1-p(x,y)}{1-q(x,y)} \right).
\]
The objective penalizes both the relative entropy of $p$ with respect to $q$ and of $1-p$ with respect to $1-q$ in a symmetric way.

\section{Main approach}
In this section we will introduce a functional having orthogonalization properties and then include it in the objective of an optimization problem regularized with the relative entropy. The optimal solution is a function in the exponential family, due to the entropy regularization (cf. \cite{georgiou2006relative}), and we use this to design an implicit function iteration. 
We then investigate the corresponding sequence and show that it converges globally for certain parameter values. After this, we will introduce a doubly stochastic optimization problem where the fix point of the sequence is the solution. Also, we consider a memory efficient  approximation for large scale problems.

\subsection{Orthogonalization functional}\label{sec:Orthogonalization}
Consider the Gromov-Wasserstein distance in the discrete case with $X=Y$ and where each point of support has mass one in section \ref{ssec:gromov_hausdorff}.
Clearly, the optimal transport plan is the identity map $\Gamma=id$. Now consider the cost of misalignment, i.e., the cost of 
a permutation $\Gamma$ that flips the pair $(\hat x,\hat y)$ where $\hat x, \hat y\in X$ and maps each of the remaining points to itself. The cost in the Gromov-Wasserstein sense of this misalignment is given by

\[G_{d_X}(\hat x,\hat y) = 4 \int_w (d_X(\hat x,w)-d_X(\hat y,w))^2 dw.\] 
In the Gromov-Wasserstein distance it is not required that $d_X$ and $d_Y$ represents distances. In fact, any relationship between points can be used and the optimization problem depends only on the differences defined by these relationships. For example, by taking a diffusion kernel and dividing it by the square root of its stationary distribution,  $\tilde p/\sqrt{\tilde{\pi}}$, gives a misalignment cost that can be expressed directly in terms of the Diffusion distance (\ref{eq:diffusionDistance}) as 
\[G_{\tilde p/\sqrt{\tilde{\pi}}}(\hat x,\hat y) = k D_p(\hat x,\hat y)^2,
\]
for some constant $k$. By integrating over all $x,y$ with a weight $p$ that the permutation occurs leads to the orthogonalization functional

\begin{equation}
O_{\tilde{p}}(p) = \int_{x,y} p(x,y)\ddist{\tilde{p}} (x,y)dxdy,
\label{eq:Polarization}
\end{equation}
where 
$$\ddist{p}(\hat x,\hat y):=D_p^1(\hat x,\hat y)^2= \|p(x,\cdot)-p(y,\cdot)\|_{L^2(X,d\mu/ \pi)}^2.$$
This term then penalizes weights $p$, which has large values for pairs $(x,y)$ which have large misalignment cost according to the objective in the Gromov-Wasserstein cost.  

To better understand the effect of this functional, consider the ideal case with perfect clusters.  Let $\tilde p$ correspond to a partitioning $X= \bigcup_{i=1}^n S_i$. This means that $\tilde p(x,y)$ is small whenever $x\in S_i$ and $y\in S_j$ with $i\neq j$, and since $\tilde p$ is normalized we have that  $\ddist{\tilde{p}}(x,y)$ is large. 
Consequently, if $p(x,y)$ is not consistent with this partition, i.e., contains large values for $x\in S_i$ and $y\in S_j$ with $i\neq j$, then the integrand $p(x,y)\ddist{\tilde{p}}(x,y)$ is large.  Therefore, by minimizing the functional \eqref{eq:Polarization} with respect to $p$, the corresponding optimal $p$ will be consistent with the partitioning in $\tilde p$. 

\subsection{Introductory optimization problem}\label{sec:Main}

Let $X\subset \mathbb{R}^n$ be a compact set and let 
\begin{align*}
E(X^2) = \left\{p\in L^\infty(X^2): \int_y p(x,y) dy = 1, \quad\forall x \in X,  0<\epsilon < p(x,y) \quad \forall x,y \in X\right\}
\end{align*} 
be the space of diffusion map kernels which are densities bounded away from zero and with normalized mass for each $x$. 

Now, consider the joint functional with the orthogonalization functional (\ref{eq:Polarization}) weighted with $c_2 \ge 0$ regularized with the relative entropy. That is, for a fix $\tilde{p}$ we consider the optimization problem
\begin{equation}
\begin{split}
\min_{p\in E(X^2)} \int_{x,y} p(x,y) \log \frac{p(x,y)}{q(x,y)} dxdy + c_2 O_{\tilde{p}}(p),
\label{eq:Optimizationproblem}
\end{split}
\end{equation}

Here $q$ represents a prior, which needs to be positive, for example a kernel obtained from a Diffusion map (\ref{eq:markovprocess}). In this case the total cost becomes a combination of the prior and the diffusion distance.

The Lagrangian of the problem is 
\begin{equation}
\begin{split}
\mathcal{L}(p,\xi) = \int_{x,y\in X} \left(p\log p + p(-\log q+c_2\ddist{\tilde{p}})+ \xi (x) \left(\mu(X)^{-1}-p(x,y)\right) \right) dxdy,
\end{split}
\label{eq:Lagrangian}
\end{equation}
where $\xi$ is the dual variable and $\mu(X)$ is the volume of the space. Minimizing with respect to $p$  gives the extreme point
\[p(x,y)^* = q(x,y) \exp(- c_2 \ddist{\tilde{p}}(x,y)+\xi(x)),\]
and solving the dual problem gives $\xi$ so that the constraints hold, i.e.,  
\[p(x,y)^* = q(x,y) \frac{\exp(- c_2 \ddist{\tilde{p}}(x,y))}{D_{\tilde{p}}(x)}\]
where
$D_{\tilde{p}}(x) = \int_y q(x,y)\exp(- c_2 \ddist{\tilde{p}}(x,y))dy.$

If $\tilde{p} \in E(X^2)$ then $0 <\epsilon< \exp(- c_2 \ddist{\tilde{p}}(x,y)) \le 1$ for all $x$ and $y$. Further, if $q \in E(X^2)$ then  
$0<D_{\tilde{p}}(x) $ for all $x\in X$. Thus, since $p^* \in E(X^2)$ is a stationary point in the relative interior of the domain and the problem is convex we can conclude that this is the optimal solution. Furthermore, $p^*$ is in fact the product of $q$ and a Diffusion map corresponding to the diffusion kernel $\tilde p$. Since the prior $\tilde p$ is arbitrary, one can consider linking it to $p$. 
This gives motivation to understand the operator $f:E(X^2) \rightarrow E(X^2)$ defined as
\begin{equation}
\label{eq:iteration2function}
f(p)(x,y) =q(x,y) \frac{\exp(-c_2 \ddist{p}(x,y))}{D_p(x)}.
\end{equation}
In particular, it is of interest to understand when this operator is a contractive map with the existence and possible uniqueness of solutions to $p=f(p)$.

\subsection{An iterative method with sequence of priors}

We propose an iterative method to solve the stationary point of \eqref{eq:iteration2function}, first by initialization of $p_0$ and then consider the sequence $\{p_n\}_{n}$ generated by
\begin{equation}
p_n(x,y) = q(x,y) \frac{\exp(-c_2 \ddist{p_{n-1}}(x,y))}{D_{p_{n-1}}(x)}.
\label{eq:iteration2}
\end{equation}
A natural initialization of the method is to use $p_0=q$, as this is the stationary point when $c_2 = 0$. 

Note that since $E(X^2)$ is not a complete subset of $L^\infty(X^2)$, we do not know if there is an accumulation point of this sequence in $E(X^2)$. However, initial studies show that convergence can be guaranteed in some cases.
In particular, we show that for any $q \in E(X^2)$ there exist a $c_2> 0$ so that if $p_0 \in E(X^2)$, then the 
sequence $\{p_n\}_n$ is bounded away  from 0 and infinity. Furthermore, we show that in this case $\{p_n\}_n$ is also a Cauchy sequence, resulting in that there is a unique fixed point.  

\begin{theorem} Given constants $0<\alpha_q<\beta_q$ there are constants $0<c_2$ and $0<\alpha_p < \beta_p$ such that whenever $p$ and $q$ satisfies  $\alpha_q \le q(x,y)\le \beta_q$ and  $\alpha_p \le p(x,y) \le \beta_p$ for all $x,y$, we have that  the function $f(p)$ defined by (\ref{eq:iteration2function}) satisfies $\alpha_p \le f(p)(x,y) \le \beta_p$. 
\end{theorem}
\begin{proof}
The stationary distribution $\pi$ of $p$ is clearly bounded by $\alpha_p \le \|p^{-1}\|_\infty \le \pi \le \|p\|_\infty \le \beta_p$ since $\pi(y) = \int_x \pi(x) p(x,y)dx$.

Further, let  $p$ be so that $L_p(x,y) = \int_w(p(w,x)-p(w,y))^2/\pi(w) dw \le E$ by setting $E = \mu(X) \|p\|_\infty^2 /\|p^{-1}\|_\infty$.

Then a low bound of the denominator $D_p$ is
\begin{align*}
\min_x D_p(x) =
\min_x \int_y q(x,y)\exp(-c_2L_p(x,y)) dy \ge \mu(X) \alpha_q \exp(-c_2E).
\end{align*}
Then
\[\|f(p)\|_\infty \le \frac{\beta_q}{\mu(X)\alpha_q \exp(-c_2E)}.\]
Also, a high bound of the denominator $D_p$ is
\[\max_x D_p(x) =\max_x \int_y q(x,y)\exp(-c_2L_p(x,y)) dy \le \mu(X)\beta_q.\]
Then
\[\|f(p)^{-1}\|_\infty \le \frac{\alpha_q \exp(-c_2E)}{\mu(X)\beta_q}.\]
Also $\|f(p)-p\|_\infty \le \|f(p)\|_\infty + \|p\|_\infty < \infty$ for all $c_2 \le C < \infty$.
\end{proof} 

We follow the approach of Hu and Kirk with local radial contractions as presented in \cite{hu1978local}.  
Let $(X, d)$ be a metric space, $k \in (0, 1)$, and $g:X\rightarrow X$. If
for each $x\in X$ there exists a neighborhood $N(x)$ of $x$ such that for each
$u, v \in N(x)$, $d(g(u), g(v)) < kd(u, v)$, then $g$ is said to be locally contractive (with constant $k$). If it is assumed only that $d(g(u), g(x)) < kd(u, x)$ for all
$u\in N(x)$, then g is said to be a local radial contraction.
\begin{theorem}
\label{thm:localcontraction}
If $q$ is continuous, and $p$ is uniformly continuous on all the elements on a finite partition of $X$, there is a $c_2 > 0$ so that the iterative function \eqref{eq:iteration2function} is a local radial contraction.
\end{theorem}

\begin{proof}
By \eqref{eq:iteration2function} we have that
\[\| f(p_0)-f(p)\| = \left\| \frac{q\exp(-c_2 L_{p+\delta u})}{D_{p+\delta u}}-f(p)\right\|
\]
for which we seek a Taylor expansion. Let 
\[\hat{p} = p+u\]
for which $u$ is such that the dominant eigenvalue of $\lambda(\hat{p}) = 1$ and $p+u$ has a compact resolvent. Then we seek a stationary distribution $\hat{\pi}$ such that
\[\hat{\pi} = \int_x \hat{\pi}(x)(p(x,y)+u(x,y))dx.\]
Clearly, $u = 0$ gives $\hat{\pi} = \pi$, so we consider shifting the operator so that $\hat{\pi} = \pi + v$, where $\int_x \pi(x)v(x)dx = 0$, and then 
\begin{align*}
v(y)+\pi(y) =  \int_x (v(x)+\pi(x))(p(x,y)+u(x,y))dx\Rightarrow\\
v(y) = \int_x \pi(x) u(x,y) dx + \int_x v(x)( p(x,y)+u(x,y))dx.
\end{align*}
This is a Fredholm equation of the second type. Existence of a solution follows from the analysis of such \cite{fredholm1903classe}, and is common in complete textbooks in analysis.

Suppose there is a finite partition $\mathcal{P} = \{S_k\}_{k=1}^K$ of $X$ and $p$ is uniformly piecewise continuous on $\mathcal{P}^2$. Also take $\psi$ that is is piecewise uniformly continuous and bounded on $\mathcal{P}$.
Since $p$ is uniformly continuous on all partitions it is also bounded with $\sup_{x,y} |p(x,y)| = P$. Then 
\begin{align*}
|g(\psi)(y)| &= \left| \int_{x\in X} p(x,y) \psi(x)dx \right| \le P  \int_{x \in X}| \psi(x)| dx \le P \mu(X) \|\psi\|_\infty.
\end{align*}
From this it follows that $\|g(\psi)\|_\infty  = \sup_y |g(\psi)(y)| \le P \mu(X) \|\psi\|_\infty$. Thus, $g(\psi)$ is finitely bounded.  

To check continuity, for every $\epsilon > 0$ there exist $\delta > 0$ such that $|p(x,y)-p(x',y')| < \epsilon$ whenever $\|x-x'\|+\|y-y'\| < \delta$ in each region $S_k$. Therefore $g(\psi)(y) = \int_X p(x,y) \psi(x) dx$ is a piece wise uniformly continuous function for which 
\begin{align*}
|g(\psi)(y)-g(\psi)(y')|
&\le \sum_{k=1}^K \int_{S_k} |p(x,y)-p(x,y')||\psi(x)|dx \\
&\le \epsilon \int_X |\psi(x)|dx \le \epsilon \|\psi \|_\infty \mu(X)
\end{align*}
when $\|y-y'\| < \delta$. 

Take a sequence $\{\psi_k\}_k$ of piecewise uniformly continuous functions. Since this space is compact it has a subsequence that converge to $\psi$. We denote this convergent subsequence $\{\psi_{n}\}_{n}$. By above, $\{g(\psi_{n})\}_n$ are equicontinuous. Therefore 
$| g(\psi_n)(x)-g(\psi)(y)| \le | g(\psi_n)(x)-g(\psi_n)(y)| + |g(\psi_n) (y)-g(\psi)(y)| \le \epsilon\|\psi_n\|_\infty \mu(X) + |g(\psi_n) (y)-g(\psi)(y)|$ when $\|x-y\| \le \delta$. 

The second term
\begin{align*}
|g(\psi_n) (y)-g(\psi)(y)| &\le \int_x |p(x,y)| |\psi_n(x)-\psi(x)|\\
& \le P \int_x |\psi_n(x)-\psi(x)|dx \le P \mu(X) \|\psi_n-\psi\|_\infty.
\end{align*}
Since $\psi_n$ converge to $\psi$ there is an $N$ so that $\|\psi_N-\psi\|_\infty < \epsilon_2$ for every $\epsilon_2>0$, and compactness of $g$ follows.

Compact linear operators have a compact and at most countable spectrum, with an accumulation point in $0$. Also, strictly positive integral kernels have a single eigenvalue of maximum modulus. Thus, the perturbation of the stationary distribution is unique.   

Since there exist a finite set of nontrivial solutions and especially the unique stationary distribution $\hat{\pi} = \int_x \hat{\pi}(x)(p(x,y)+u(x,y))dx$, there exist a solution as long as $\int_y \int_x \pi(x)u(x,y)dx \hat{\psi}_1(y)dy = 0$. Indeed, as $\psi_1(y) = 1$, $\int_y u(x,y) dy = 0$ generates yet another Markov chain $p+u$, the above holds.

Since $\int_y u(x,y) \psi_1(y) dy= 0$, it is possible to use the Neumann series to compute $v$ , even though $\|p\|=1$ by letting $g_0(y) := \int_y u(x,y)\pi(x)dx$ and $g_k(y) := \int_x(p(x,y)+u(x,y))g_{k-1}(x)dx$ giving
\[v(x) = \sum_{i=0}^\infty g_k(x).\]
Thus, $v$ can be approximated to a bounded linear map 
\[v(y) = \int_x u(x,y) \pi(x) dx + \int_{x,z} p(z,y) u(x,z) \pi(x)dxdz+\mathcal{O}(\|u\|^2).\]

Normalizing and scaling $u = \delta \bar{u}$ gives $\hat{\pi}$ as the expression of a bounded linear map $H_1$
\[\hat{\pi} = \pi + \delta H_1(u)+ \mathcal{O}(\delta^2) \]
where 
\[
H_1(\bar{u})(y)  = \int_x \bar{u}(x,y)\pi(x)dx + \int_{x,z} p(z,y)\bar{u}(x,z)\pi(x)dxdz.\]

By Taylor expansion we have 
\[
\frac{1}{\hat{\pi}}= \frac{1}{\pi + \delta H_1(u)} = \frac{1}{\pi} - \delta \frac{H_1(u)}{\pi^2}+ \mathcal{O}(\delta^2).
\]

Expanding the square in $L$ and using the Taylor expansion of $1/\hat{\pi}$ gives
\begin{align*}
L_{p+\delta u} =& L_{p}
+ \delta \int_w 2  (p(x,w)-p(y,w))(u(x,w)-u(y,w))/\pi(w) dw\\
&-\delta \int_w (p(w,x)-p(w,y))^2 H_1(u)(w)/\pi^2(w) dw+ \mathcal{O}(\delta^2)\\
 =& L_p+\delta H_2(u) + \mathcal{O}(\delta^2).
\end{align*}

where $H_2(u)$ is a bounded linear operator of $u$. As an intermediate step we have the Taylor expansion of the exponential 
\begin{align*}
\exp(-c_2 L_{p+\delta u}) = \exp(-c_2 L_p)[ 1-c_2\delta H_2(u)+ \mathcal{O}(\delta^2)]. 
\end{align*}
Finally, the Taylor expansion of 
\begin{align*}
\frac{1}{D_{p+\delta u}(x)} =& \frac{1}{D_p(x)} + c_2 \delta \frac{\int_y q(x,y)\exp(-c_2L_p(x,y)) H_2(u)(x,y) dy}{D_p(x)^2}+ \mathcal{O}(\delta^2)\\
=& \frac{1}{D_p(x)} + c_2 \delta\frac{H_3(u)(x)}{D_p(x)} +\mathcal{O}(\delta^2)
\end{align*}
where $H_3(u)$ is a bounded linear operator of $u$. Putting all this into the norm gives
\begin{align*}
\| f(\hat{p})-f(p)\|= \|f(p)[c_2 \delta H_3(u)(x)-c_2\delta H_2(u) + \mathcal{O}(\delta^2)]\|.
\end{align*}
Thus
\begin{align*}
\frac{\|df\|}{\|dp\|} = \lim_{\delta \rightarrow 0}  \frac{\| f(p+\delta u)-f(p)\|}{\delta \|u\|} = c_2 \|f(p)[H_3(u)-H_2(u)]\|,
\end{align*}
and by using the infinity norm we obtain
\[ 
\frac{\|df\|_\infty}{\|dp\|_\infty} \le c_2 \|f(p)\|_\infty \|H_3(u)-H_2(u)\|_\infty.
\]

From the previous theorem there exist a $c_2> 0$ so that $\|f(p)\|_\infty  \le M< \infty$ for all $p\in E(X^2)$. Since $H_2$ and $H_3$ are bounded linear operators, we can always find a $c_2$ so that $\frac{\|df\|}{\|dp\|} \le m < 1$ and hence $f$ is a local radial contraction.
\end{proof}

\begin{theorem}
Given $q\in E(X^2)$, there is a $c_2> 0$ so that the sequence $\{p_n\}_n$ generated by \eqref{eq:iteration2function} is a Cauchy sequence.
\end{theorem}

\begin{proof}

Let $p_{n+1} = p_n+u$. By integrating over the curve parametrized by $\alpha$ in $f(p_n+\alpha u)$ we get an upper bound for the direct distance
\begin{align*}
\|p_{n+2}-p_{n+1}\|_\infty &= \|f(p_{n+1})-f(p_n)\|_\infty\\
& \le \int_{\alpha=0}^1 \|df(p_n+\alpha u)\|_\infty d\alpha\\
&\le  m \int_{\alpha = 0}^1 \|dp(p_n+\alpha u)\|_\infty d\alpha\\
& = m\|p_{n+1}-p_n\|_\infty.
\end{align*} 
Hence, $\|p_{n+1}-p_n\| \le \|p_1-p_0\| m^n$, where there exist a $c_2 > 0$ so that $m < 1$. 
From the first theorem we see that for all $p \in E(X^2)$ and $c_2 < \infty$ we have that $f(p) \in E(X^2)$, hence $\|p_1-p_0\| \le C < \infty$.
Thus, the sequence $\{p_n\}_n$ generated by \eqref{eq:iteration2function} is a Cauchy sequence for any starting $p\in E(X^2)$.
\end{proof}

Thus, we can conclude that for a $c_2$ small enough there is a unique fixed point.

For large values of $c_2$ we have seen examples of more erratic behavior, e.g., with several solutions corresponding to different clustering which may depend on the starting point. In future work we are considering adapting $c_2$ which corresponds to the resolution of the method in order to simultaneously find several clustering relevant at different scales.

\section{Generalizations and efficient computations}

\subsection{Version with doubly stochastic constraints}\label{sec:DoublyStochastic}

In the approach in Section~\ref{sec:Main} the optimization problem is changed in each iteration, which can be seen as a type of bootstrapping that updates $\tilde p$ to refine the clustering in each iteration. However, by not only requiring the solutions to be stochastic, i.e., $p\in E(X^2)$, but doubly stochastic and symmetric, i.e., $p$ belongs to 
\begin{align*}
E_2(X^2)=\! \bigg\{&p\in L^\infty( X^2)\!\mid 
\int_{y\in X} p(x,y) dy= 1, \; \forall x \in X, \;0 <\epsilon< p(x,y)=p(y,x), \;  \forall x,y \in X \bigg\},
\end{align*}
the objective function can be adjusted so that the corresponding update can be seen as minimization of a non-convex objective function. In particular, consider the orthogonalization functional  
\begin{align*}
&\hat{O}(p) =\! \int_{x,y,w}\!\! \left(\!p(x,y)(p(x,w)-p(y,w))^2 - 2p^2(x,y)p(x,w)\!+ \!\frac{4}{3}  p(x,w)p(w,y)p(x,y)\!\right)\!dxdydw,
\end{align*}
which has the Gateaux derivative $$d\hat{O}(p;u) = \int_{x,y,w} \left((p(x,w)-p(y,w))^2 dw\right) u(x,y) dxdy = \int_{x,y} \ddist{p}(x,y) u(x,y) dxdy $$
in any direction satisfying $\int_x u(x,y) dx = \int_y u(x,y) dy = 0$, required in order for $p+u\in E_2(X^2)$.
Then by considering the optimization problem
\[ 
\min_{p\in E_2} \int_{x,y \in X} p \log \frac{p}{q} dxdy + c_2 \hat{O}(p)
\]
and applying Lagrange relaxation it can be shown that the optimal solution $p^*$ is of the form 
$p^*(x,y) =q(x,y) \exp(-c_2 \ddist{p^*}(x,y))/(D_1(x)D_2(y))$, and 
thus the stationary points correspond to fixed points of the mapping
\[
f(p)(x,y) =q(x,y) \frac{\exp(-c_2 \ddist{p}(x,y))}{D_1(x)D_2(y)}
\]
where $D_1$ and $D_2$ are the normalizing factors that makes $p$ doubly stochastic. The factors $D_1$ and $D_2$ can be obtained with Sinkhorn's method, made popular on entropy regularized optimal transport described by Cuturi in \cite{cuturi2013sinkhorn}.
We can also see that if $q,p\in E_2(X^2)$, then $D_1 = D_2$ and $f(p)\in E_2(X^2)$.
Also in this case, the mapping can be shown to be contractive for $c_2$ sufficiently small. In fact, for any compact subset of $E_2(X^2)$, if the constant $c_2>0$ is sufficiently small then the problem is convex on the subset.
This can be seen by considering the second derivative of the Lagrangian in direction $u$
\begin{align*}
d^2\mathcal{L}(p,u) = &\int_{x,y} u(x,y)^2/p(x,y)dxdy - 4 c_2 \int_{x,y,w} p(x,w) u(x,y) u(y,w) dxdydw
\end{align*}
and noting that the first term is dominating the second term on any compact subset if
 $c_2$ is small enough.  Hence the extreme point is a unique minimizer. In the case of single stochastic diffusion, a simple explicit expression for the rest term has not been found, that gives this self-similar structure. 

\subsection{Truncation of eigenvectors for efficient computations}\label{sec:trunc}
For large data sets, it is computationally expensive to compute and store $q$ and the iterates $\{p_n\}_n$ since all the elements are $N\times N$ matrices. 
However, by utilizing \eqref{eq:diffusionDistance} we may express $\ddist{\tilde p}$ in terms of its diffusion coordinates, which correspond to the eigenvalues of $\tilde p$. This allows for fast approximate computations using only the dominating eigenvalues. 

Note that at the fixed point we can express $p$ in terms of the bi-orthogonal vectors as in Section  \ref{ssec:diffusion_maps}
\begin{align*}
\sum_{i=1}^\infty \lambda_i\varphi_i(y)\psi_i(x)=p(x,y) = q(x,y) \frac{\exp(-c_2 \ddist{p}(x,y))}{D_p(x)}
\end{align*}
Integrating against $\psi_j(y)$, we have by bi-orthogonality that $\int_y \psi_j(y) \varphi_i(y)dy=\delta_{i,j}$ where $\delta_{i,j}$ is the Kronecker delta function, and thus  
\[
\lambda_i \psi_i(x) =  \frac{\int_{y\in X} \psi_i(y) q(x,y)\exp(-c_2 \ddist{p}(x,y)) dy}{D_p(x)}.
\]

Further, noting that $\ddist{p}(x,y) = \sum_{i=1}^\infty \lambda_i^2 (\psi_i(x)-\psi_i(y))^2$ may be expressed in terms of the eigenvectors it can be approximated by 
\begin{align*}
\ddist{p,M}(x,y) &= \sum_{i=1}^M \lambda_i^2 (\psi_i(x)-\psi_i(y))^2,
\end{align*}
and the full fixpoint iteration can be approximately performed using only the largest $M$ eigenvectors. 

Let $\tildeddist{{p,M}}(x,y) = \sum_{i=M+1}^\infty \lambda_i^2 (\psi_i(x)-\psi_i(y))^2$, and consider the perturbed $q$ given by
\[
q_M(x,y) = \frac{q(x,y)\exp(-c_2 \tildeddist{p,M}(x,y))}{\int_y q(x,y)\exp(-c_2 \tildeddist{{p,M}}(x,y)) dy}.
\]

Also, let an operator $T_M$ define the truncation $p_M^* = T_M(p^*) = \sum_{i=1}^M \lambda_i \psi_i(x)\varphi_i(y)$ denote the truncated sum of the original $p^*$.
Then $T_M(f(p^*_M,q_M)) = T_M(p^*) = p^*_M$ and the iterated sequence has converged to this instead, as the normalization of $q_M$ cancels out in the normalization of $p$ in $f$. Thus, a truncation $p_M^*$ of $p^*$ can be considered as a perturbation $q_M$ of $q$. Existence and uniqueness follows but for the perturbed $q$. Suppose that the eigenvalues are ordered in descending order $|\lambda_1| > |\lambda_2| > |\lambda_3|,...$ then the perturbation of $q$ goes to zero as $M \rightarrow \infty$.

The connection between $\Psi$ to the more commonly used normalized left and right eigenvectors in the unweighted $L^2(X)$-norm as used in most software packages exist but with a scaling factor $N_l \ge 0$ for each eigenvalue. 
Suppose the eigenvector $v_l = N_l \psi_l$. Since $\|\phi_l\|_{L^2(X)} = 1$ is normalized we have that 
$\phi_l = \sqrt{\pi}\psi_l= \frac{\sqrt{\pi}v_l}{\|\sqrt{\pi}v_l\|_{L^2(X)}}$. Also, since $\|v_l\|_{L^2(X)}=1$ is normalized we have that
$ v_l =\psi_l/\|\psi_l\|_{L^2(X)}=\psi_l/\sqrt{\int_x \phi_l(x)^2/\pi(x)dx}= \psi_l \|v_l\sqrt{\pi}\|_{L^2(X)}=\psi_l N_l$
giving a convenient factor from $v_l$ to $\psi_l$, namely 
$\psi_l = v_l/\|\sqrt{\pi}v_l\|_{L^2(X)}$.

\section{Examples}\label{ssec:Examples}

\subsection{Clustering of high dimensional data}\label{sec:ex1}
As a starting example consider three clusters embedded in noise given by 
\[q_0 := 10 S_{3N\times 3N} +\begin{pmatrix}
S_{N\times N} & 0& 0\\
0 &S_{N\times N} & 0\\
0&0&S_{N\times N}
\end{pmatrix}.
\]
Here, $S_{N\times N}$ is a symmetric $N\times N$ matrix with uniform random numbers in $[0,1]$. When $q$ is 
processed using the method the resulting diffusion coordinates of the limit point $p^*$ is presented in  Figure~\ref{Noise}. When using the truncation of eigenvectors as presented in section \ref{sec:trunc} the noise is removed and the separation between clusters is increased, which gives means of an algorithmic way to cluster the data based on distance.  

\begin{figure}[ht]
\vskip 0.2in
\begin{center}
\centerline{\includegraphics[width=0.8\columnwidth]{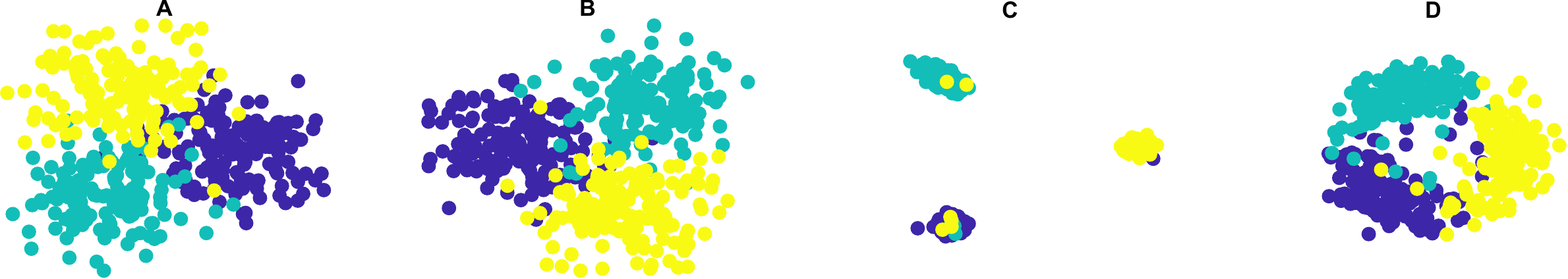}}
\caption{ Clustering three diverse groups with added random noise. $A$ shows the two dominant dimensions of the diffusion coordinate of the unprocessed data, $B$ shows the two dominant dimensions of the diffusion coordinate of the orthogonalized $p$, showing that the method does not add structure to the data. $C$ shows the two dominant dimensions of the processed optimal $p$ using the truncation (20 dimensions) of eigenvectors as presented in section \ref{sec:trunc} giving a separation between the artificial clusters. $D$ shows a t-SNE for comparison. }
\label{Noise}
\end{center}
\vskip -0.2in
\end{figure}

\subsection{Clustering of gene therapy vector expressions}\label{sec:ex2}
Adeno Associated Viruses (AAVs) can be used as vectors in the delivery of genes into cells to achieve a therapeutic effect. Novel therapeutics include treatment of e.g. genetic blindness and muscle disorder. When manufacturing such vectors, integrity, packaging performance of the gene cargo and purity of the product is of highest importance. There are only a few available analytic methods to measure the success of packaging the correct cargo into a vector, which makes development hard and quality control expensive. No cargo, truncated cargo or other than expected cargo may be packaged into the particles during the assembly process in the manufacturing host cell.

The particles have an icosahedral structure of about 25 nm in diameter. Using electron microscopy, the morphology such as the density variation due to cargo inside the particle and capsid structure can be observed. Images were collected and 61x61 sub images approximately centering on single particles were extracted. To evaluate the performance of the proposed method, each particle was classified by an expert virologist in three classes: full, empty, and uncertain, see figure \ref{AAVs}. The distance was chosen to be the direct distance between the pixel values, i.e. the squared Euclidean distance defining the diffusion kernel $q(i,j) := \exp(-c_{1,i} \|x_i-x_j\|_2^2)$. Here $c_{1,i}$ was set so that the number of neighbors were $N= 200$, i.e. $\mu(\{j: c_{1,i} \|x_i-x_j\|_2^2 \le 1/\sqrt{2}\}) = N$. The diffusion coordinates of the optimal $p^*$ by (\ref{eq:iteration2function}) were used as input to a t-SNE (see Figure~\ref{AAVs}) compared with a t-SNE and U-Map of the original data. The spectrum of $p^*$ having $\sum_i \lambda_i \approx 10$ reveals the number of clusters. Using k-means to separate the clusters, the performance of the method to separate the classes provided by the virologist is highly accurate, and the method also separates different viewpoints of the icosahedral structure which manifests in a hexagonal pattern when projected on a two-dimensional image seen in figure \ref{AAVs2}.

\begin{figure}[]
\vskip 0.2in
\begin{subfigure}[b]{0.2\textwidth}
\centering
\includegraphics[angle=90,width=0.35\columnwidth]{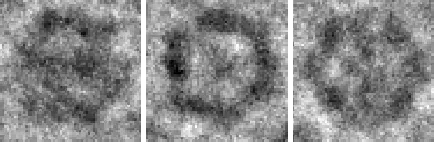}
\subcaption{AAV particles.}
\end{subfigure}
\begin{subfigure}[b]{0.75\textwidth}
\centerline{\includegraphics[width=0.95\columnwidth]{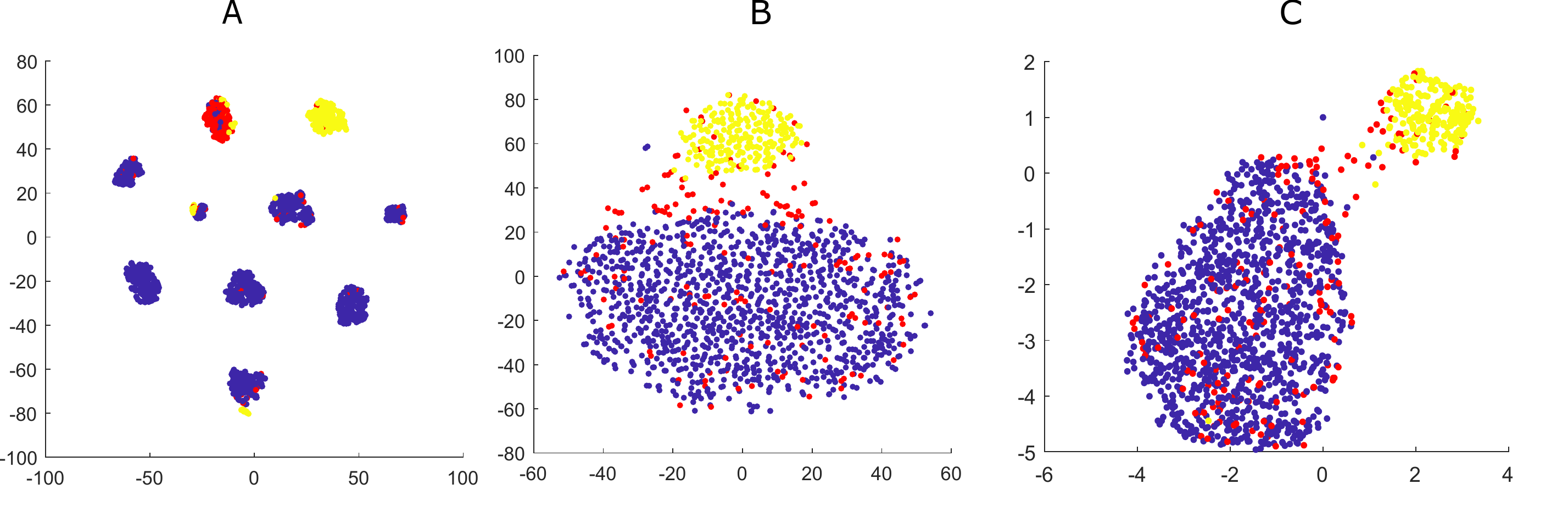}}
\subcaption{Visualization of data. A: proposed approach. B: t-SNE. C: U-Map.}
\end{subfigure}

\caption{ (a) Examples of AAV particles with different degree of packaging. From top: Particle with uncertain degree of packaging, Empty particle, Full particle. The imaging is slightly under-focused to increase contrast. (b) Clustering of AAV particles. (A) shows a t-SNE plot of the orthogonalized data using the 200 nearest neighbors for scaling. The coloring is based on expert evaluation of packaging in empty (yellow), full (blue) and an uncertain class (red). (B) shows a t-SNE plot of the original data. The expert classification is not preserved. (C) shows a U-Map plot of the original data. The expert classification is not preserved.
}
\label{AAVs}
\vskip -0.2in
\end{figure}

\begin{figure}[]
\vskip 0.2in
\begin{center}
\begin{subfigure}[b]{0.32\textwidth}
\includegraphics[width=0.99\columnwidth]{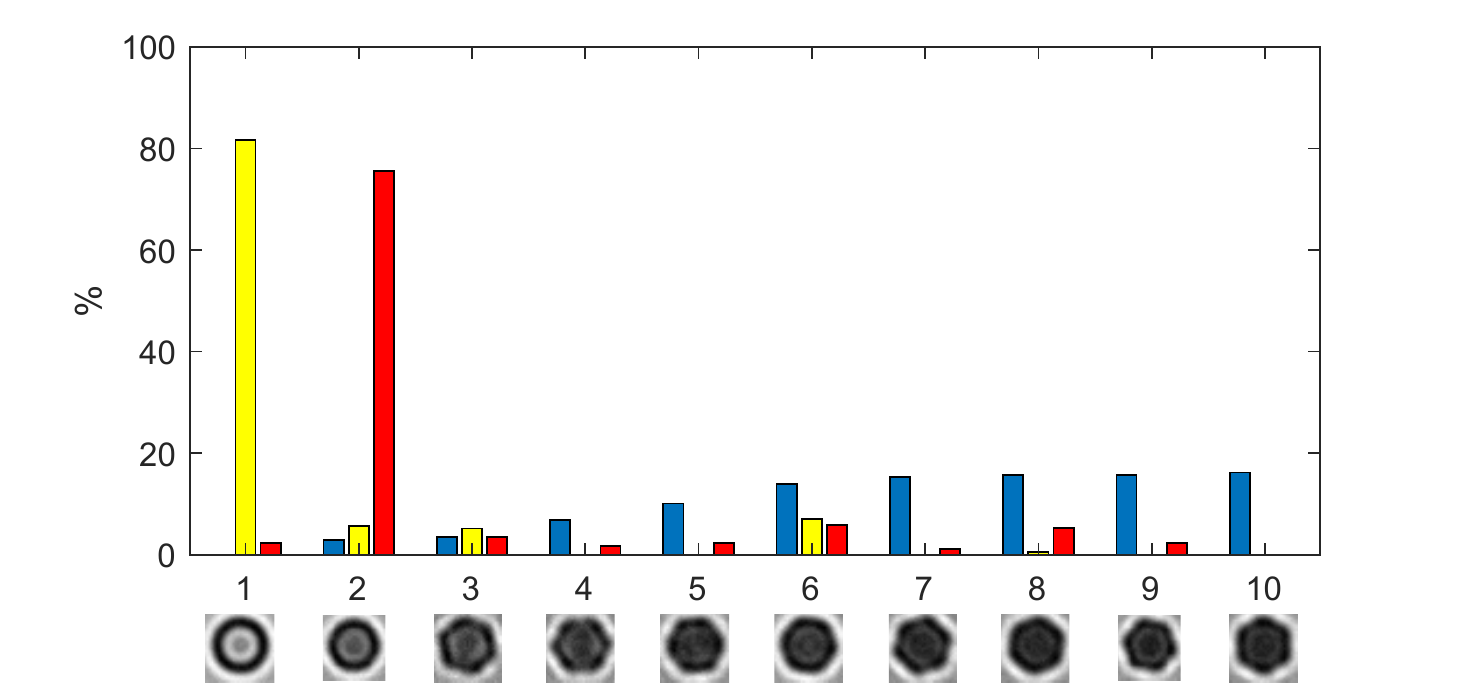}
\subcaption{Proposed approach}
\end{subfigure}
\begin{subfigure}[b]{0.32\textwidth}
\centerline{\includegraphics[width=0.99\columnwidth]{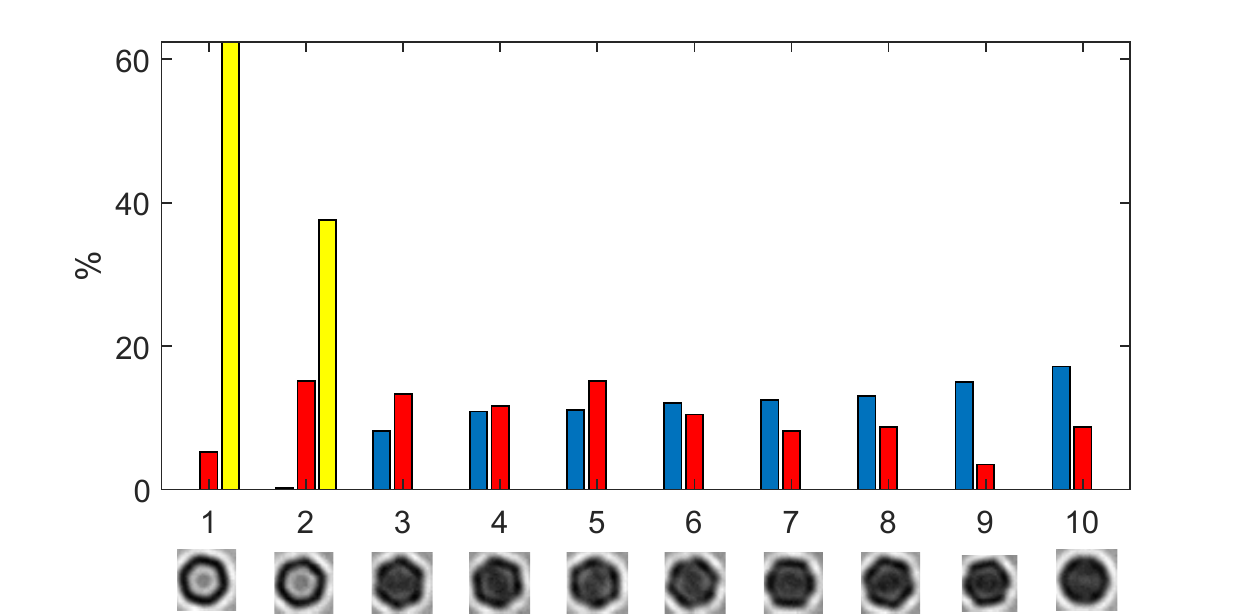}}
\subcaption{t-SNE}
\end{subfigure}
\begin{subfigure}[b]{0.32\textwidth}
\centerline{\includegraphics[width=0.99\columnwidth]{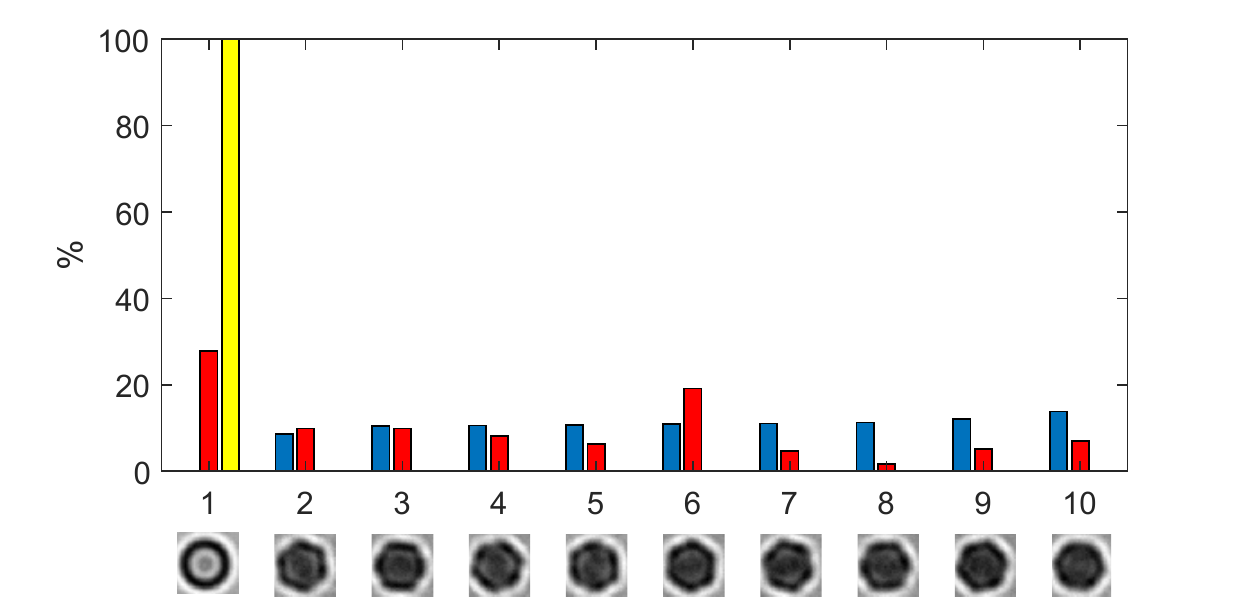} }
\subcaption{U-Map}
\end{subfigure}

\caption{Bar plot of the presence of different expert classed particles in the clusters and a mean image below each class to give structural information using k-means on the processed data. 
In (a) cluster 1 is dominated by the empty particles (yellow) and cluster 2 is dominated by the particles with uncertain packaging (red). Particles with dense core are divided in several clusters indicating different orientation of the icosahedral structure of the particle (blue). 
In (b) and (c) the particles with uncertain packaging are distributed between all clusters.
}
\label{AAVs2}
\end{center}
\vskip -0.2in
\end{figure}

\subsection{Comparison with state of the art and discussion}

In this section we compare the proposed approach with t-SNE and U-Map. In order to provide a quantitative comparison, we cluster the output from Sections \ref{sec:ex1} and \ref{sec:ex2} using K-means and compare the results using three clustering performance evaluation metrics, namely 
 Purity \cite{manning2010introduction}, Normalized mutual information (NMI) \cite{strehl2002cluster} and Adjusted Rand Index (ARI) \cite{hubert1985comparing}, see table \ref{AAVex} for results. For this class of problems our proposed method gives a better result, which can be linked to the robustness of representing the data in terms of the transition probabilities directly rather than coordinates generating the transition probabilities.

\begin{table}[H]
\caption{Comparative indices of the proposed method, t-SNE and U-Map for the two examples provided in Section \ref{ssec:Examples}.}
\label{AAVex}
\vskip 0.15in
\begin{center}
\begin{small}
\begin{sc}
\begin{tabular}{lccc|lccc}
\toprule
Example \ref{sec:ex1} & ARI & NMI & Purity &Example \ref{sec:ex2} & ARI & NMI & Purity \\
\midrule
Proposed & 0.9214         & 0.8789            & 0.9733 & Proposed & 0.1103    & 0.3809       & 0.9257\\
t-SNE     & 0.8777         & 0.8255            & 0.9578 & t-SNE  & 0.0790            & 0.3027            & 0.8866  \\
U-MAP     & 0.3306         & 0.2862            & 0.7178  & UMAP  & 0.0975            & 0.3136            & 0.8866 
\\
\bottomrule

\end{tabular}
\end{sc}
\end{small}
\end{center}
\vskip -0.1in
\end{table}

\section{Future work}
In this paper we have discussed an orthogonalization process of data which can be used in clustering and improved visualization. The scale harmonics is to be further investigated as several scales may be of interest at the same time. Also, the orthogonalization can potentially be used to accelerate the Gromov-Wasserstein distance calculations and hence be used in matching geometric objects of rigid structures or the geometry of clusters of conceptual classes.

\section{Acknowledgements and Disclosure of Funding}
This work was funded by Vironova AB and the Swedish innovation agency Vinnova through AAVNova (2019-00103) and GeneNova (2021-02640). The data for the example was provided by Vironova.

\medskip
{
\small

\bibliography{refs}

\begin{thebibliography}{10}

\bibitem{beauchemin2015affinity}
Mario Beauchemin.
\newblock On affinity matrix normalization for graph cuts and spectral
  clustering.
\newblock {\em Pattern Recognition Letters}, 68:90--96, 2015.

\bibitem{burago2001course}
Dmitri Burago, Iu~D Burago, Yuri Burago, Sergei Ivanov, Sergei~V Ivanov, and
  Sergei~A Ivanov.
\newblock {\em A course in metric geometry}, volume~33.
\newblock American Mathematical Soc., 2001.

\bibitem{carlsson2009topology}
Gunnar Carlsson.
\newblock Topology and data.
\newblock {\em Bulletin of the American Mathematical Society}, 46(2):255--308,
  2009.

\bibitem{coifman2006diffusion}
Ronald~R Coifman and St{\'e}phane Lafon.
\newblock Diffusion maps.
\newblock {\em Applied and computational harmonic analysis}, 21(1):5--30, 2006.

\bibitem{cuturi2013sinkhorn}
Marco Cuturi.
\newblock Sinkhorn distances: Lightspeed computation of optimal transport.
\newblock {\em Advances in neural information processing systems},
  26:2292--2300, 2013.

\bibitem{edwards1975structure}
David~A Edwards.
\newblock The structure of superspace.
\newblock In {\em Studies in topology}, pages 121--133. Elsevier, 1975.

\bibitem{fredholm1903classe}
Ivar Fredholm.
\newblock Sur une classe d’{\'e}quations fonctionnelles.
\newblock {\em Acta mathematica}, 27(1):365--390, 1903.

\bibitem{georgiou2006relative}
Tryphon~T Georgiou.
\newblock Relative entropy and the multivariable multidimensional moment
  problem.
\newblock {\em IEEE Transactions on Information Theory}, 52(3):1052--1066,
  2006.

\bibitem{gromov1981groups}
Mikhael Gromov.
\newblock Groups of polynomial growth and expanding maps.
\newblock {\em Publications Math{\'e}matiques de l'Institut des Hautes
  {\'E}tudes Scientifiques}, 53(1):53--78, 1981.

\bibitem{hinton2002stochastic}
Geoffrey Hinton and Sam~T Roweis.
\newblock Stochastic neighbor embedding.
\newblock In {\em NIPS}, volume~15, pages 833--840. Citeseer, 2002.

\bibitem{hu1978local}
Thakyin Hu and WA~Kirk.
\newblock Local contractions in metric spaces.
\newblock {\em Proceedings of the American Mathematical Society},
  68(1):121--124, 1978.

\bibitem{hubert1985comparing}
Lawrence Hubert and Phipps Arabie.
\newblock Comparing partitions.
\newblock {\em Journal of classification}, 2(1):193--218, 1985.

\bibitem{johnson1967hierarchical}
Stephen~C Johnson.
\newblock Hierarchical clustering schemes.
\newblock {\em Psychometrika}, 32(3):241--254, 1967.

\bibitem{landa2021doubly}
Boris Landa, Ronald~R Coifman, and Yuval Kluger.
\newblock Doubly stochastic normalization of the {G}aussian kernel is robust to
  heteroskedastic noise.
\newblock {\em SIAM journal on mathematics of data science}, 3(1):388--413,
  2021.

\bibitem{macqueen1967some}
James MacQueen et~al.
\newblock Some methods for classification and analysis of multivariate
  observations.
\newblock In {\em Proceedings of the fifth Berkeley symposium on mathematical
  statistics and probability}, volume~1, pages 281--297. Oakland, CA, USA,
  1967.

\bibitem{manning2010introduction}
Christopher Manning, Prabhakar Raghavan, and Hinrich Sch{\"u}tze.
\newblock Introduction to information retrieval.
\newblock {\em Natural Language Engineering}, 16(1):100--103, 2010.

\bibitem{mcinnes2018umap}
Leland McInnes, John Healy, and James Melville.
\newblock Umap: Uniform manifold approximation and projection for dimension
  reduction.
\newblock {\em arXiv preprint arXiv:1802.03426}, 2018.

\bibitem{memoli2007use}
Facundo M{\'e}moli.
\newblock On the use of gromov-hausdorff distances for shape comparison.
\newblock In {\em Point-based Graphics 2007: Eurographics/IEEE VGTC Symposium
  Proceedings, Prague, Czech Republic, September 2-3, 2007}. The Eurographics
  Association, 2007.

\bibitem{memoli2008gromov}
Facundo M{\'e}moli.
\newblock Gromov-{H}ausdorff distances in {E}uclidean spaces.
\newblock In {\em 2008 IEEE Computer Society Conference on Computer Vision and
  Pattern Recognition Workshops}, pages 1--8. IEEE, 2008.

\bibitem{ng2001spectral}
Andrew Ng, Michael Jordan, and Yair Weiss.
\newblock On spectral clustering: Analysis and an algorithm.
\newblock {\em Advances in neural information processing systems}, 14, 2001.

\bibitem{pardalos1994quadratic}
Panos~M Pardalos, Henry Wolkowicz, et~al.
\newblock {\em Quadratic Assignment and Related Problems: DIMACS Workshop, May
  20-21, 1993}, volume~16.
\newblock American Mathematical Soc., 1994.

\bibitem{ros2019hierarchical}
Fr{\'e}d{\'e}ric Ros and Serge Guillaume.
\newblock A hierarchical clustering algorithm and an improvement of the single
  linkage criterion to deal with noise.
\newblock {\em Expert Systems with Applications}, 128:96--108, 2019.

\bibitem{strehl2002cluster}
Alexander Strehl and Joydeep Ghosh.
\newblock Cluster ensembles---a knowledge reuse framework for combining
  multiple partitions.
\newblock {\em Journal of machine learning research}, 3(Dec):583--617, 2002.

\bibitem{van2008visualizing}
Laurens Van~der Maaten and Geoffrey Hinton.
\newblock Visualizing data using t-{SNE}.
\newblock {\em Journal of machine learning research}, 9(11), 2008.

\bibitem{wang2010learning}
Fei Wang, Ping Li, and Arnd~Christian Konig.
\newblock Learning a bi-stochastic data similarity matrix.
\newblock In {\em 2010 IEEE International Conference on Data Mining}, pages
  551--560. IEEE, 2010.

\bibitem{zass2006doubly}
Ron Zass and Amnon Shashua.
\newblock Doubly stochastic normalization for spectral clustering.
\newblock {\em Advances in neural information processing systems}, 19, 2006.

\end{thebibliography}
\bibliographystyle{plain}

}

\end{document}